%% file: paper.tex
\documentclass[runningheads]{llncs}

\usepackage[utf8]{inputenc} % allow utf-8 input
\usepackage[T1]{fontenc}    % use 8-bit T1 fonts
\usepackage{hyperref}       % hyperlinks
\usepackage{url}            % simple URL typesetting
\usepackage{booktabs}       % professional-quality tables
\usepackage{amsfonts}       % blackboard math symbols
\usepackage{nicefrac}       % compact symbols for 1/2, etc.
\usepackage{microtype}      % microtypography
\usepackage{graphicx,graphics}
\usepackage{latexsym,amsmath,amssymb,amsthm}
\usepackage{algorithm2e,multirow,verbatim,color}
\usepackage{subcaption}

\usepackage{textcomp}
\usepackage{stmaryrd}

\usepackage{cleveref}
\crefformat{footnote}{#2\footnotemark[#1]#3}
\usepackage{relsize}

\graphicspath{{plots/}}
\DeclareGraphicsExtensions{.pdf}

\usepackage{pgf}
\usepackage{tikz}

\usetikzlibrary{arrows,decorations.pathmorphing,decorations.footprints,fadings,calc,trees,mindmap,shadows,decorations.text,patterns,positioning,shapes,matrix,fit}
\usetikzlibrary{arrows,shadows,backgrounds}
\usetikzlibrary{arrows.meta}
\usetikzlibrary{positioning,fit}
\usetikzlibrary{automata}
\usetikzlibrary{matrix}
\usetikzlibrary{shapes.symbols,shapes.misc,shapes.arrows}
\usetikzlibrary{matrix,chains,scopes,decorations.pathmorphing}
\usetikzlibrary{bayesnet}

\input{defs}
\begin{document} 

\input{preamble}

\maketitle

\input{abs}

\input{intro}
\input{prelim}
\input{xlc}

\input{res}

\input{conc}
\bibliographystyle{abbrv}
\bibliography{refs,xai}

%\clearpage
\input{appendix}

\end{document}

%% file: defs.tex
\makeatletter
\newtheorem*{rep@theorem}{\rep@title}
\newcommand{\newreptheorem}[2]{%
\newenvironment{rep#1}[1]{%
 \def\rep@title{#2 \ref{##1}}%
 \begin{rep@theorem}}%
 {\end{rep@theorem}}}
\makeatother

\newreptheorem{theorem}{Theorem}
\newreptheorem{proposition}{Proposition}
\newreptheorem{lemma}{Lemma}

\newcommand{\fml}[1]{{\mathcal{#1}}}

\newcommand{\tn}[1]{\textnormal{#1}}

\newcommand{\mbf}[1]{\ensuremath\mathbf{#1}}
\newcommand{\mbb}[1]{\ensuremath\mathbb{#1}}
\newcommand{\msf}[1]{\ensuremath\mathsf{#1}}
\newcommand{\tsf}[1]{{\textsf{\small #1}}}

\newcommand{\lvt}{\mathbf{t}}
\newcommand{\lvf}{\mathbf{f}}

\DeclareMathOperator*{\limply}{\rightarrow}
\DeclareMathOperator*{\argmax}{argmax}

\DeclareMathOperator*{\prob}{Pr}
\DeclareMathOperator*{\lprob}{lPr}

\definecolor{gray}{rgb}{.4,.4,.4}
\definecolor{midgrey}{rgb}{0.5,0.5,0.5}
\definecolor{middarkgrey}{rgb}{0.35,0.35,0.35}
\definecolor{darkgrey}{rgb}{0.3,0.3,0.3}
\definecolor{darkred}{rgb}{0.7,0.1,0.1}
\definecolor{midblue}{rgb}{0.2,0.2,0.7}
\definecolor{darkblue}{rgb}{0.1,0.1,0.5}
\definecolor{defseagreen}{cmyk}{0.69,0,0.50,0}

\newcommand{\nnote}[1]{\noindent$\llbracket$\textcolor{darkred}{nina}: \emph{\textcolor{blue}{#1}}$\rrbracket$}

\newcommand{\jnoteF}[1]{}

\newcommand{\TopBlankLine}{\vspace*{3.5pt}}
\newcommand{\BotBlankLine}{\vspace*{1.5pt}}

\newcommand{\fnsqz}{0}

%% file: preamble.tex
\title{Explaining Naive Bayes and Other Linear \\
  Classifiers with Polynomial Time and Delay}

\titlerunning{Explaining Linear Classifiers}

\author{%
  Joao Marques-Silva\inst{1} \and
  Thomas Gerspacher\inst{1} \and
  Martin C. Cooper\inst{2,1} \and\\
  Alexey Ignatiev\inst{3} \and
  Nina Narodytska\inst{4}
}

\authorrunning{J.~Marques-Silva et al.}

\institute{
  ANITI, Universit\'{e} de Toulouse, France \\
  \email{joao.marques-silva@irit.fr}, 
  \email{thomas.gerspacher@irit.fr}
  \and
  IRIT, Universit\'{e} de Toulouse III, France,
  \email{cooper@irit.fr}
  \and
  Monash University, Australia, 
  \email{alexey.ignatiev@monash.edu}
  \and
  VMware Research, CA, USA,
  \email{nnarodytska@vmware.com}
}

%% file: abs.tex
\begin{abstract}
  Recent work proposed the computation of so-called PI-explanations of
  Naive Bayes Classifiers (NBCs)~\cite{darwiche-ijcai18}.
  PI-explanations are subset-minimal sets of feature-value pairs that
  are sufficient for the prediction, and have been computed with
  state-of-the-art exact algorithms that are worst-case exponential in
  time and space.
  In contrast, we show that the computation of one PI-explanation for
  an NBC can be achieved in log-linear time, and that the same result
  also applies to the more general class of linear classifiers.
  Furthermore, we show that the enumeration of PI-explanations can be
  obtained with polynomial delay.
  Experimental results demonstrate the performance gains of the new
  algorithms when compared with earlier work. The experimental results
  also investigate ways to measure the quality of heuristic
  explanations.
\end{abstract}
%
%
\begin{comment}
%
\begin{abstract}
  %
  %Naive Bayes Classifiers (NBCs) find widespread use.
  %
  Recent work proposed worst-case exponential time and space
  algorithms for computing non-heuristic PI-explanations of Naive
  Bayes Classifiers (NBCs), where PI-explanations are defined as
  subset-minimal sets of feature values that are sufficient for the
  prediction.\nnote{It might be hard for a reviewer to understand this sentence as 'non-heuristic PI-explanations' is not a commonly used notion.}
  %
  In contrast, this paper shows that the computation of one
  PI-explanation for an NBC can be achieved in log-linear time, and
  that the same result also applies to the more general class of
  linear classifiers. Furthermore, the paper shows that the
  enumeration of PI-explanations can be obtained with polynomial
  delay.
  %
  Experimental results demonstrate the performance difference between
  the new algorithms and earlier work. The experimental results also
  compare the new algorithms with heuristic approaches for producing
  explanations.
  %
\end{abstract}
%
\end{comment}

%% file: intro.tex
\section{Introduction} \label{sec:intro}

Approaches proposed in recent years for computing explanations of
Machine Learning (ML) models can be broadly characterized as
\emph{heuristic} or \emph{non-heuristic}\footnote{There is a large
  body of recent work on explaining ML models. Example recent
  overviews include~\cite{%
  %klein-ieee-is17a,klein-ieee-is17b,cotton-ijcai07-xai,muller-dsp18,klein-ieee-is18a,klein-ieee-is18b,berrada-ieee-access18,mencar-ipmu18,hlupic-mipro18,klein-corr18,
  pedreschi-acmcs19,xai-bk19,muller-xai19-ch01,miller-aij19,miller-acm-xrds19,anjomshoae-aamas19,russell-fat19a,zhu-nlpcc19,klein-corr19}.}.
Heuristic approaches denote those providing \emph{no} formal
guarantees on their results. In contrast, non-heuristic approaches
\emph{do} provide some sort of formal guarantee(s) on their results,
usually at the cost of increased computational complexity.
Among the heuristic approaches for finding explanations, two have been
studied in greater detail. One line of work focuses on devising
model-agnostic linear approximations of the underlying
model~\cite{guestrin-kdd16,lundberg-nips17}.
Another line of work is exemplified by Anchor~\cite{guestrin-aaai18},
and targets the computation of a set of feature-value pairs associated
with a given instance as a way of explaining the prediction.
To date, all non-heuristic methods have focused on computing sets of
feature-value pairs that are sufficient for the
prediction~\cite{darwiche-ijcai18,inms-aaai19,darwiche-aaai19,darwiche-arxiv-2020}%
\footnote{%
  Earlier work imposed the additional restriction of considering
  boolean-valued features. Clearly, non-boolean features can be
  binarized, e.g.\ with the one hot encoding, at the cost of adding
  additional features.}.
%\nnote{Last statement is not very clear. Anchor accepts categorical
%features. We cannot feed one-hot encoded features as Anchor will not
%be able to perturb features.}. -> Joao: fixed
%
Moreover, in terms of formal guarantees, \cite{darwiche-ijcai18}
studies two distinct definitions of explanations.
%, assuming binary features.
%
%between PI-explanations (which are also considered in other
%work~\cite{inms-aaai19}) and CM-explanations.
%\nnote{I think, we need to spell out PI/CM first.} -> Joao: fixed
%
A PI-explanation represents a subset-minimal set of feature values
that entails the outcome of the decision function for the predicted
class whatever the values of the other features 
(i.e.\ it represents a \emph{prime implicant} of the outcome of the
decision function). PI-explanations have also been studied
under the name of abductive explanations~\cite{inms-aaai19}.
In contrast, and assuming binary features, an MC-explanation is a
cardinality-minimal set of equal-valued features that entails the
outcome of the decision function. %for the predicted class.
Non-heuristic approaches are model-based, and so earlier work
specifically considered Naive-Bayes
Classifiers (NBCs) and Latent-Tree Classifiers
(LTCs)~\cite{darwiche-ijcai18,darwiche-arxiv-2020}, Bayesian Network
Classifiers~\cite{darwiche-aaai19,darwiche-arxiv-2020}, and Neural
Networks~\cite{inms-aaai19}.

In the concrete case of computing (non-heuristic) PI-explanations for
NBCs, earlier work~\cite{darwiche-ijcai18} proposed algorithms that
are worst-case exponential in both time and space.
In contrast, in this paper we propose a novel non-heuristic solution
for computing PI-explanations of NBCs and other linear
classifiers~\footnote{%
  In fact, the paper considers a generalization of linear classifiers,
  that accommodates both real-valued and categorical features, which
  serves to streamline the presentation. This generalization will be
  referred to as an \emph{eXtended Linear Classifier} (XLC).
  %Although linear classifiers can be considered as
  %interpretable~\cite{guestrin-kdd16}, it is unclear how
  %PI-explanations could be computed when the linear classifier
  %works with real-value or categorical features, or both.
},
which exhibits two fundamental advantages over earlier work.
First, the paper shows that computing PI-explanations for NBCs (but
also for any linear classifier) is in P, by proposing a log-linear
algorithm for computing one smallest size PI-explanation.
Second, the paper proposes a polynomial (log-linear) delay algorithm
for enumerating the PI-explanations of NBCs (and also of any linear
classifier).
Furthermore, the paper presents an experimental evaluation of
different approaches for explaining NBCs with PI-explanations, 
including the heuristic solutions computed by
Anchor~\cite{guestrin-aaai18} and SHAP~\cite{lundberg-nips17}%
\footnote{It should be noted that for linear classifiers (including
  NBCs), heuristic explanation approaches based on linear
  approximations, such as those provided by LIME~\cite{guestrin-kdd16}
  or SHAP~\cite{lundberg-nips17}, can be regarded as uninteresting,
  since the model is itself linear. Nevertheless, aiming for coverage,
  we opt to include also results for SHAP.}.
Moreover, although (real-valued) linear classifiers can be viewed
as interpretable~\cite{guestrin-kdd16}, this does not equate with
computing PI-explanations, particularly when features are
categorical. 
To the best of our knowledge,
%and besides non-heuristic worst-case exponential
%approaches~\cite{darwiche-ijcai18},
proving the (polynomial) complexity of computing PI-explanations for
linear classifiers (including NBCs) closes an open problem.
%whether PI-explanations %and CM-explanations could be computed
%efficiently.

The paper is organized as follows.
\autoref{sec:prelim} introduces the concepts and notation used
throughout the paper.
\autoref{sec:xlc} introduces XLCs (a simple extension of linear
classifiers (LCs)), and develops a new approach for computing, in
polynomial time, one PI-explanation for XLCs.
\autoref{sec:xlc} also proposes a polynomial delay algorithm for the
enumeration of PI-explanations of XLCs.
\autoref{sec:res} compares dedicated approaches for explaining
NBCs~\cite{darwiche-ijcai18} with the algorithms proposed in this
paper, but also with the explanations produced by heuristic
approaches.
The paper concludes in~\autoref{sec:conc}.

\jnoteF{ToDo: confirm whether SHAP computes linear approximation.}

%% file: prelim.tex
\section{Preliminaries} \label{sec:prelim}

\paragraph{Explanations of ML models.}
We consider a classification problem with two classes
$\fml{K}=\{\oplus,\ominus\}$, defined on a set of features (or
attributes) $e_1,\ldots,e_n$, which will be represented by their
indices $\fml{E}=\{1,\ldots,n\}$. The features can either be
real-valued or categorical. For real-valued features, we have
$\lambda_i\le{e_i}\le\mu_i$, where $\lambda_i$, $\mu_i$ are given
lower and upper bounds. For categorical features, we have
$e_i\in\{1,\ldots,d_i\}$.
A concrete assignment to the features referenced by $\fml{E}$ is
represented by an $n$-dimensional vector $\mbf{a}=(a_1,\ldots,a_n)$,
where $a_j$ denotes the value assigned to feature $j$, represented by
variable $e_j$, such that $a_j$ is taken from the domain of $e_j$. The
set of all $n$-dimensional vectors denotes the \emph{feature space}
$\mbb{E}$.
Given a classifier with features $\fml{E}$, a \emph{decision
  function}~\cite{darwiche-ijcai18} is a mapping from the feature
space to the set of classes, i.e.\ 
$\tau:\mbb{E}\to\fml{K}$. %%(\mbf{e})
For example, for a linear classifier, the decision function picks 
$\oplus$ if $\sum_iw_ie_i>0$, and $\ominus$ if $\sum_iw_ie_i\le0$.
Given $\mbf{a}\in\mbb{E}$, with $\tau(\mbf{a})=\oplus$, we consider
the set of feature literals of the form $(e_i=a_i)$, where $e_i$
denotes a variable and $a_i$ a constant.
A PI-explanation~\cite{darwiche-ijcai18} is a subset-minimal set
$\fml{P}\subseteq\fml{E}$, denoting feature literals, such that,
%%feature -value pairs (taken from $\mbf{a}$) such that,
%
\begin{equation} \label{eq:xpi01}
\forall(\mbf{e}\in\mbb{E}).\bigwedge\nolimits_{j\in\fml{P}}(e_j=a_j) \ \limply \ \tau(\mbf{e})=\oplus
\end{equation}
is true. 
Alternatively, we can represent~\eqref{eq:xpi01} as a rule:
\begin{equation} \label{eq:xpi02}
  \begin{array}{cccc}
    \textbf{IF} & \bigwedge\nolimits_{j\in\fml{P}}(e_j=a_j) & \textbf{THEN} &
    \tau(\mbf{e})=\oplus\\
  \end{array}
\end{equation}
(The same definitions apply in the case of class $\ominus$ 
(given $\mbf{a}\in\mbb{E}$, with $\tau(\mbf{a})=\ominus$).)
%~\footnote{%
%MC-explanations~\cite{darwiche-ijcai18} could be defined similarly,
%but are beyond the scope of this paper.}
%
%Moreover, PI-explanations are also referred to as abductive
%explanations~\cite{inms-aaai19}.

\paragraph{Naive Bayes Classifier (NBC).}
NBCs~\cite{duda-bk73} can be viewed as special cases of Bayesian Network
Classifiers (BNCs)~\cite{friedman-ml97}, that make strong
conditional independence assumptions among the features.
%
% NBCs are generally recognized as one of the most widely used
% classifiers~\cite{wu-kis08,wu-bk09}.
%
Graphically, NBCs are represented as depicted in~\autoref{fig:ex01}
for a concrete example.
Given some evidence $\mbf{e}$   
(in our case, this is an
assignment to the features), the predicted class is given
by:
\begin{equation} \label{eq:nbc1}
  \tau(\mbf{e}) = \argmax\nolimits_{c\in\fml{K}}\left(\prob(c|\mbf{e})\right)
\end{equation}
%
%%Ties are broken by picking one of the classes. We consider
%%$\pi=\oplus$.
%
It is well known that $\prob(c|\mbf{e})$ can be computed as follows:
$\prob(c|\mbf{e})=\frac{\prob(c,\mbf{e})}{\prob(\mbf{e})}$.
However, $\prob(\mbf{e})$ is constant for every $c\in\fml{K}$. Hence, 
\eqref{eq:nbc1} can be rewritten as follows:
\begin{equation} \label{eq:nbc2}
  \tau(\mbf{e}) = \argmax\nolimits_{c\in\fml{K}}\left(\prob(c,\mbf{e})\right)
\end{equation}
Finally, assuming features to be mutually conditional
independent,~\eqref{eq:nbc2} can be rewritten as follows:
\begin{equation} \label{eq:nbc3}
  \tau(\mbf{e}) =
  \argmax\nolimits_{c\in\fml{K}}\left(\prob(c)\times\prod\nolimits_i\prob(e_i|c)\right)
\end{equation}
A standard transformation is to apply logarithms, thus getting:
\begin{equation} \label{eq:nbc4}
  \tau(\mbf{e}) =
  \argmax\nolimits_{c\in\fml{K}}\left(\log{\prob(c)}+\sum\nolimits_i\log{\prob(e_i|c)}\right)
\end{equation}
Also, if $\prob(e_i|c)=0$, then we use instead a sufficiently large 
negative value $\mbb{M}$~\cite{park-aaai02}%
~\footnote{This section follows~\cite{park-aaai02} throughout. An
  alternative would be to use Laplace smoothing~\cite{DBLP:books/daglib/0021593}.}, i.e.\ we pick
$\max(\mbb{M},\log(\prob(e_i|c)))\in[\mbb{M},0]$.
(A simple solution is to use the sum of the logarithms of all the
non-zero probabilities plus some $\epsilon<0$.) 
For simplicity, i.e.\ to work with positive values, we can add a
sufficiently large positive threshold $\mbb{T}$ to each probability,
to serve as a reference, thus obtaining:
\begin{equation} \label{eq:nbc5}
  \tau(\mbf{e}) =
  \argmax\nolimits_{c\in\fml{K}}\left((\mbb{T}+\log{\prob(c)})+\sum\nolimits_i(\mbb{T}+\log{\prob(e_i|c)})\right)
  %\Uppi
\end{equation}
(For example, we can set $\mbb{T}$ to the complement of the negative
value with the largest absolute value.)
Also for simplicity, we use the notation
$\lprob(\alpha)\triangleq\mbb{T}+\max(\mbb{M},\log(\prob(\alpha)))$.
%and also let 
%$\nu(\alpha)\triangleq\lprob(\alpha)+\sum_i\lprob(e_i|\alpha)$.
%%$\log|_{\mbb{M}}(\prob(e_i|c))$....\\
%%
%We will write $\nu_{\fml{E}}(\alpha)$ to denote that $\nu$ is taken
%over the evidence features in $\fml{E}$.

\paragraph{Running Example.}
Consider the NBC shown
in~\autoref{fig:ex01}~\footnote{%
      This example of an NBC is adapted from~\cite[Ch.10]{barber-bk12},
      with some of the conditional probabilities changed.}.
%
%
%\begin{example} \label{ex:ex01}
%  We consider the NBC from~\autoref{fig:ex01}.
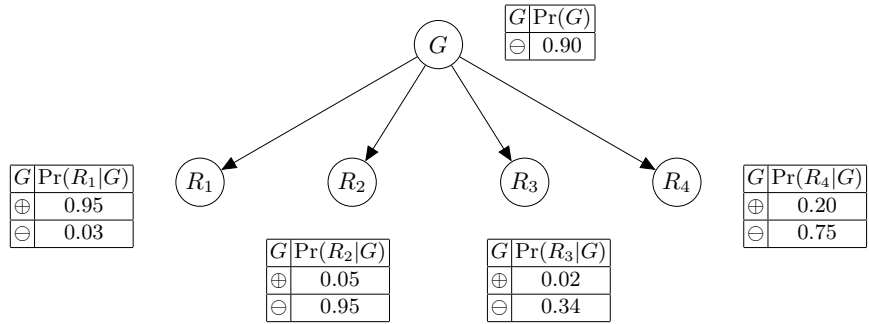
\begin{figure}[t]
  \begin{center}
    \input{./texfigs/ex01}
  \end{center}
  \caption{Running example.} \label{fig:ex01} 
\end{figure}
The features are the random variables $R_1$, $R_2$, $R_3$ and $R_4$.
Each $R_i$ can take values $\lvt$ or $\lvf$ denoting, respectively,
whether a listener likes or not that radio station.
%%~\footnote{%
%%In the original example, each random variable $R_i$ is a radio
%%station, and a listener may ($\lvt$) or may not ($\lvf$) like
%%that radio station.}.
%
Random variable $G$ denotes an \tsf{age} class, which can take
values \tsf{Y} and \tsf{O}, denoting \tsf{young} and \tsf{older}
listeners, respectively. Using the notation proposed earlier, we will
use $\oplus$ for \tsf{Y} and $\ominus$ for \tsf{O}. We also associate
$\oplus$ with $1$ or $\lvt$ and $\ominus$ with $0$ or $\lvf$.
In general we have,
\begin{align} \label{eq:ex01a}
  \prob(G,R_1,R_2,&R_3,R_4)= \nonumber \\
  & \prob(G)\times\prob(R_1|G)\times\prob(R_2|G)\times\prob(R_3|G)\times\prob(R_4|G)
\end{align}

Considering the assignment
$(G,R_1,R_2,R_2,R_3)=(\oplus,\lvt,\lvf,\lvt,\lvf)$,
and using $g$ to denote $G=\oplus$, $r_i$ to denote $R_i=\lvt$ and $\neg{r_i}$ to denote
$R_i=\lvf$,~\eqref{eq:ex01a}
%(alternatively, we can also replace $T$ with 1 and $F$ with 0,
%i.e.~$(A,R_1,R_2,R_2,R_3)=(\oplus,1,0,1,0)$),~\eqref{eq:ex01a}
can be written as follows:
%
%\begin{align}
%  \prob&(A=T,R_1=T,R_2=F,R_3=T,R_4=F) = \nonumber \\
%  & \prob(A=\oplus)\times\prob(R_1=T|A=T)\times\prob(R_2=F|A=T)\times\prob(R_3=T|A=T)\times\prob(R_4=F|A=T) \nonumber %\\
%\end{align}
%%
%which we write instead (in a much more succinct way) as follows:
\[
\prob(g,r_1,\neg{r_2},r_3,\neg{r_4})=\prob(g)\times\prob(r_1|g)\times\prob(\neg{r_2}|g)\times\prob(r_3|g)\times\prob(\neg{r_4}|g)
\]
Let us consider 
$\mbf{a}=(R_1,R_2,R_3,R_4)=(\lvt,\lvf,\lvt,\lvf)$.
%\nnote{We use $a$ as a random variable value and $\mbf{a}$ as  an
%assignment. Might be confusing for the reader.} -> Joao: fixed
%What are the values of 
%$\prob(a|r_1,\neg{r_2},r_3,\neg{r_4})$ and
%$\prob(\neg{a}|r_1,\neg{r_2},r_3,\neg{r_4})$?
%
%
%We can compute these posterior probabilities as follows:
%%
%\begin{align}
%  \prob(a|r_1,\neg{r_2},r_3,\neg{r_4})&=
%  \alpha\times\prob(a)\times\prob(r_1|a)\times\prob(\neg{r_2}|a)
%  \times\prob(r_3|a)\times\prob(\neg{r_4}|a) = \nonumber \\
%  & = \alpha\times0.10\times0.95\times0.95\times0.02\times0.80 \nonumber \\
%  & = 0.001444\times\alpha \nonumber \\[10pt]
%%\end{align}
%%
%%\begin{align}
%  \prob(\neg{a}|r_1,\neg{r_2},r_3,\neg{r_4})&=
%  \alpha\times\prob(\neg{a})\times\prob(r_1|\neg{a})\times\prob(\neg{r_2}|\neg{a})
%  \times\prob(r_3|\neg{a})\times\prob(\neg{r_4}|\neg{a}) = \nonumber \\
%  & = \alpha\times0.90\times0.03\times0.18\times0.34\times0.08 \nonumber \\
%  & = 0.00013219\times\alpha \nonumber
%\end{align}
%
%Thus, the predicted class is $\pi=\oplus$.
%
%
Since all probabilities are strictly positive,
%\nnote{I would say non-equal to zero or strictly positive} -> Joao: fixed
we set $\mbb{M}$ to a very large negative (irrelevant) value. In
addition, we set $\mbb{T}$ to a value above the complement of the
logarithm of the smallest probability (i.e.\ 0.02), e.g\ we can set
%$\mbb{T}=4 > -\log_e(0.02)$. 
$\mbb{T}=4 > -\log(0.02)$. 
\begin{figure}[t]
  \begin{subfigure}[t]{1.0\linewidth}
    \centering\input{./tabs/ex01a}
    \caption{Computing $\lprob(\oplus|\mbf{a})$}
  \end{subfigure}

  \begin{subfigure}[t]{1.0\linewidth}
    \centering\input{./tabs/ex01b}
    \caption{Computing $\lprob(\ominus|\mbf{a})$}
  \end{subfigure}
  \caption{Deciding prediction for
    $\mbf{a}=(\lvt,\lvf,\lvt,\lvf)$} \label{fig:ex02}
\end{figure}
%\begin{figure}[t]
%  \begin{subfigure}[t]{1.0\linewidth}
%    \centering\input{./tabs/ex02a}
%    \caption{Computing $\nu(\oplus)$}
%  \end{subfigure}
%
%  \begin{subfigure}[t]{1.0\linewidth}
%    \centering\input{./tabs/ex02b}
%    \caption{Computing $\nu(\ominus)$}
%  \end{subfigure}
%  \caption{Deciding prediction for $\mbf{a}=(1,0,1,0)$} \label{fig:ex02}
%\end{figure}
%
Using~\eqref{eq:nbc5}, we get the values shown in~\autoref{fig:ex02}.
As can be concluded, the prediction will be $\oplus$.
Observe that neither the value of $\mbb{M}$ nor of $\mbb{T}$ affect
the prediction.
%
%
%\begin{table}[t]
%  \begin{center}
%    \input{./tabs/ex02b}
%  \end{center}
%\end{table}

%
%\end{example}

%\paragraph{Case study.}

%\begin{figure}[t]
%  \begin{subfigure}[t]{1.0\linewidth}
%    \centering\input{./tabs/ex01a}
%    \caption{Computing $\nu(\oplus)$}
%  \end{subfigure}
%
%  \begin{subfigure}[t]{1.0\linewidth}
%    \centering\input{./tabs/ex01b}
%    \caption{Computing $\nu(\ominus)$}
%  \end{subfigure}
%  \caption{Deciding prediction for $\mbf{a}=(1,0,1,0)$} \label{fig:ex02}
%\end{figure}

%\begin{figure}[t]
%  \centering\input{./tabs/ex01c}
%  \caption{Configuration of ACC -- Temporary}
%\end{figure}

%% file: texfigs/ex01.tex
\scalebox{0.9125}{
\begin{tikzpicture} %[x=1.7cm,y=1.8cm]
  % Nodes
  \node[latent]                     (G)      {$G$};   %
  \node[latent,below left=1.5cm and 0.75cm of G]    (R2)     {$R_2$}; %
  \node[latent,below right=1.5cm and 0.75cm of G]   (R3)     {$R_3$}; %
  \node[latent,left=1.5cm of R2]    (R1)     {$R_1$}; %
  \node[latent,right=1.5cm of R3]   (R4)     {$R_4$}; %

  % Edges
  \edge[->] {G} {R1,R2,R3,R4} ;

  % Extra Nodes

  \node[right=0.5cm of G,yshift=5pt] (CPT0) { %,yshift=-10pt
    \begin{tabular}{|c|c|}\hline
      $G$ & $\prob(G)$ \\ \hline%\midrule%
      $\ominus$ & 0.90 \\ \hline%\bottomrule%
  \end{tabular} } ;
  
  \node[left=0.5cm of R1,yshift=-10pt] (CPT1) {
    \begin{tabular}{|c|c|}\hline
      $G$ & $\prob(R_1|G)$ \\ \hline
      $\oplus$ & 0.95 \\\hline
      $\ominus$ & 0.03 \\\hline\end{tabular} } ;
  
  \node[below=0.35cm of R2,xshift=-10pt] (CPT2) {
    \begin{tabular}{|c|c|}\hline
      $G$ & $\prob(R_2|G)$ \\ \hline
      $\oplus$ & 0.05 \\\hline
      $\ominus$ & 0.95 \\\hline\end{tabular} } ; %%???
  
  \node[below=0.35cm of R3,xshift=10pt] (CPT3) {
    \begin{tabular}{|c|c|}\hline
      $G$ & $\prob(R_3|G)$ \\ \hline
      $\oplus$ & 0.02 \\\hline
      $\ominus$ & 0.34 \\\hline\end{tabular} } ;
  
  \node[right=0.5cm of R4,yshift=-10pt] (CPT4) {
    \begin{tabular}{|c|c|}\hline
      $G$ & $\prob(R_4|G)$ \\ \hline
      $\oplus$ & 0.20 \\\hline
      $\ominus$ & 0.75 \\\hline\end{tabular} } ;

\end{tikzpicture}
}

%% file: tabs/ex01a.tex
\scalebox{0.925}{
\renewcommand{\tabcolsep}{0.35em}
\renewcommand{\arraystretch}{1.175}
\begin{tabular}{|c|ccccc|c|} \hline  %\toprule
  & $\prob(g)$ & $\prob(r_1|g)$ & $\prob(\neg{r_2}|g)$ & $\prob(r_3|g)$
  & $\prob(\neg{r_4}|g)$ & $\lprob(\oplus|\mbf{a})$
  \\ \hline  %\midrule
  $\prob(\cdot)$ & 0.10 & 0.95 & 0.95 & 0.02 & 0.80 &
  \\
  $\lprob(\cdot)$ & 1.70 & 3.95 & 3.95 & 0.09 & 3.78 & 13.47
  \\ \hline  %\bottomrule
\end{tabular}
}

%% file: tabs/ex01b.tex
\scalebox{0.925}{
\renewcommand{\tabcolsep}{0.35em}
\renewcommand{\arraystretch}{1.175}
\begin{tabular}{|c|ccccc|c|} \hline  %\toprule
  & $\prob(\neg{g})$ & $\prob(r_1|\neg{g})$ &
  $\prob(\neg{r_2}|\neg{g})$ & $\prob(r_3|\neg{g})$ &
  $\prob(\neg{r_4}|\neg{g})$ & $\lprob(\ominus|\mbf{a})$
  \\ \hline  %\midrule
  $\prob(\cdot)$ & 0.90 & 0.03 & 0.05 & 0.34 & 0.25 &
  \\
  $\lprob(\cdot)$ & 3.89 & 0.49 & 1.00 & 2.92 & 2.61 & 10.91
  \\ \hline  %\bottomrule
\end{tabular}
}

%% file: xlc.tex
\section{Explaining Extended Linear Classifiers} \label{sec:xlc}

This section first introduces Extended Linear Classifiers (XLCs) and
then details how PI-explanations can be computed for predictions of
XLCs.

%Throughout the rest of the paper, we assume a classification problem
%with two classes $\fml{K}=\{\oplus,\ominus\}$, defined on a set of
%features $E=\{1,\ldots,n\}$.

\subsection{Extended Linear Classifiers}
Let $\fml{E}$ be partitioned into $\fml{R}$ and $\fml{C}$, denoting
respectively the real-valued and the categorical features. Each
real-valued feature with index $i\in\fml{R}$ takes bounded values
$\lambda_i\le{e_i}\le\mu_i$. For each categorical feature $j\in\fml{C}$,
$e_j\in\{1,\ldots,d_j\}$.

We consider an XLC, that encompasses real-valued and categorical
features. Let,
\begin{equation} \label{eq:xlc01}
  \nu(\mbf{e})\triangleq%
  w_0 + \sum\nolimits_{i\in\fml{R}}w_ie_i +
  \sum\nolimits_{j\in\fml{C}}\sigma(e_j,v_j^1,v_j^2,\ldots,v_j^{d_j})
\end{equation}
%
%For each categorical feature $e_j$, its domain is
%$\{1,2,\ldots,d_j\}$, and
$\sigma$ is a selector function that picks the value $v_j^{r}$ iff
$e_j$ takes value $r$.
Moreover, let us define the decision function, $\tau(\mbf{e})=\oplus$
if $\nu(\mbf{e})>0$ and $\tau(\mbf{e})=\ominus$ if $\nu(\mbf{e})\le0$. 
%
%with the following decision function for the class $\oplus$:
%%
%\begin{equation} \label{eq:xlc01}
%  w_0 + \sum_{i\in\fml{R}}w_ie_i +
%  \sum_{j\in\fml{C}}\sigma(e_j,v_j^1,v_j^2,\ldots,v_i^{d_j}) > 0
%\end{equation}
%%

\paragraph{Reducing linear classifiers to XLCs.}
For a linear classifier, with only %\nnote{I would add `only'} -> Joao: fixed
real-valued features, simply set $\fml{C}=\emptyset$.
For an NBC with boolean features\footnote{%
  Given the proposed reductions, it is immediate to represent an NBC
  with categorical features as an XLC.},
we consider a different reduction with $\fml{R}=\emptyset$, starting from~\eqref{eq:nbc5}.
Moreover, the $\argmax$ operator in~\eqref{eq:nbc5} can be replaced by
an inequality, from which we get, 
%\nnote{the argmax in~\eqref{eq:nbc5} is replaced with inequality to
%get the following. Perhaps, we can mention this. Also, do we assume
%that all features binary to perform transformation
%from~\eqref{eq:nbc5} }: -> Joao: fixed
%
{\small
\begin{align}\label{eq:nbc2xlc}
  & \lprob(\oplus)-\lprob(\ominus)+\nonumber \\
  & \sum\nolimits_{i=1}^{n}(\lprob(e_i|\oplus)-\lprob(e_i|\ominus))e_i+
  \sum\nolimits_{i=1}^{n}(\lprob(\neg{e_i}|\oplus)-\lprob(\neg{e_i}|\ominus))\neg{e_i}>0
  %(1-e_i)
\end{align}
}%
The reduction is completed by setting:
$w_0\triangleq\lprob(\oplus)-\lprob(\ominus)$,
$v_j^1\triangleq\lprob(\neg{e_j}|\oplus)-\lprob(\neg{e_j}|\ominus)$,
$v_j^2\triangleq\lprob({e_j}|\oplus)-\lprob({e_j}|\ominus)$, and
$d_j\triangleq{2}$.

\begin{figure}[t]
  \begin{subfigure}[t]{1.0\linewidth}
    \begin{center}
      \input{./tabs/ex01c}
      \caption{Example reduction of NBC to XLC
        (\autoref{ex:ex01})} \label{fig:nbc2xlc}
    \end{center}
  \end{subfigure}
  
  \begin{subfigure}[t]{1.0\linewidth}
    \begin{center}
      \input{./tabs/ex01d}
      \caption{Computing $\delta_j$'s for the XLC 
        (\autoref{ex:ex02})} \label{fig:ex03}
    \end{center}
  \end{subfigure}
  \caption{Values used in the running example (\autoref{ex:ex01}
    and~\autoref{ex:ex02})} \label{fig:exs}
\end{figure}

\begin{example} \label{ex:ex01}
  \autoref{fig:nbc2xlc} shows the resulting XLC formulation for the
  example in~\autoref{fig:ex02}. We also let $\lvf$ be associated
  with value 1 and $\lvt$ be associated with value 2, and $d_j=2$.
\end{example}

\subsection{Explaining XLCs}

We now investigate how (smallest or cardinality-minimal)
PI-explanations can be computed for XLCs, and also how (minimal)
PI-explanations can be enumerated. For this, we need to assess how
\emph{free} some of the features are.
%
%\nnote{The transition for the previous section is a bit sharp here as
%it is not clear why we introduce slack:} -> Joao: fixed
%
For a given instance $\mbf{e}=\mbf{a}$, define a \emph{constant} slack
(or gap) value $\Gamma$ given by,
\begin{equation} \label{eq:xlc02}
  \Gamma\triangleq\nu(\mbf{a})=%
                     {w_0} + \sum\nolimits_{i\in\fml{R}}w_ia_i +
                     \sum\nolimits_{j\in\fml{C}}\sigma(a_j,v_j^1,v_j^2,\ldots,v_i^{d_j})
\end{equation}
i.e.\ this is the value obtained when deciding $\oplus$ to be the
picked class, given the assignment $\mbf{e}=\mbf{a}$.

We are interested in computing one
PI-explanation~\cite{darwiche-ijcai18} of an XLC, but we are also
interested in enumerating PI-explanations. 
As argued in~\autoref{sec:prelim}, this corresponds to finding a 
subset-minimal set of literals $\fml{P}\subseteq\fml{E}$
such that~\eqref{eq:xpi01} holds, or alternatively,
\begin{equation} \label{eq:xlc03}
  \forall(\mbf{e}\in\mbb{E}).\bigwedge\nolimits_{j\in\fml{P}}(e_j=a_j) \ \limply \ \left(\nu(\mbf{e})>0\right)
  %%\sum_{i\in\fml{R}}w_ie_i+\sum_{i\in\fml{C}}\sigma(e_j,v_j^1,v_j^2,\ldots,v_i^{d_j})>0
\end{equation}
under the assumption that $\nu(\mbf{a})>0$.
%
%
%We consider literals of the form $(e_i=a_i)$, where $e_i$ denotes a
%variable and $a_i$ a constant. Moreover, we can write,
%%
%\begin{equation} \label{eq:xlc03}
%\forall(\mbf{e}\in\mbb{E}).\bigwedge_{j\in\fml{E}}(e_j=a_j)\limply\left(\sum_{i=0}^{n}w_ie_i>0\right)
%\end{equation}
%%
%Alternatively, we can represent~\eqref{eq:xlc03} as a rule:
%
%\begin{equation} \label{eq:xlc04}
%  \begin{array}{cccc}
%    \textbf{IF} & \bigwedge\limits_{j=1}^{n}(e_j=a_j) & \textbf{THEN} &
%    \pi(\mbf{e})=\oplus\\
%  \end{array}
%\end{equation}
%
%Following earlier work~\cite{darwiche-ijcai18}, our goal when
%computing an explanation is to strengthen the rule~\eqref{eq:xlc04},
%by finding a subset-minimal set of literals such that the prediction
%remains unchanged.
%
%
%\[\bigwedge_{i=0}^{n}(e_i=a_i)\limply\left(\sum_{i=0}^{n}w_ia_i>0\right)\]
%Thus, our goal is to find a minimal subset of attributes
%$\fml{P}\subseteq\fml{E}$, s.t.
%%with $\fml{P}=\fml{E}\setminus\fml{U}$ s.t.
%%
%\begin{equation} \label{eq:xlc05}
%  \forall(\mbf{e}\in\{0,1\}^n).\bigwedge_{j\in\fml{P}}(e_j=a_j)\limply\left(\sum_{i=0}^{n}w_ie_i>0\right)
%\end{equation}
%%
%The conjunction of the set of literals given by $\fml{P}$ will be
%referred to as a prime implicant~\cite{darwiche-ijcai18} of the
%decision function associated with the XLC, and is referred to as
%PI-explanation.
%
%
In what follows, we partition $\fml{E}$ into $\fml{P}$ and $\fml{N}$, 
respectively the picked and the non-picked attributes from $\fml{E}$.
%
%Let $\Gamma(\fml{P},\mbf{a})$ denote the \emph{smallest} slack
%that can be achieved by allowing the attributes not in $\fml{P}$ to
%take any value, given that the literals in $\fml{P}$ are fixed by
%$\mbf{a}$. It is plain that $\Gamma(\fml{P},\mbf{a})>0$ holds
%iff~\eqref{eq:xlc05} also holds.
%

\paragraph{Categorical case.}
Let us first consider $\fml{R}=\emptyset$.
Each feature $e_j$ is assigned value $a_j$, which results in selecting
some value $v_j^{a_j}$, i.e.\ the value from the weights associated
with $e_j$ which is picked when $e_j=a_j$. %%$v_i^{r_i}$.
Thus, $\Gamma$ is computed as follows: 
$\Gamma=w_0+\sum_{j\in\fml{C}}v_j^{a_j}$.

Moreover, let $v_j^{\omega}$ denote the \emph{smallest} (or
\emph{worst-case}) value associated with $e_j$. 
Then, by letting every $e_j$ take \emph{any} value,
the \emph{worst-case} value of $\nu(\mbf{e})$ is,
\begin{equation}
  \Gamma^{\omega}=w_0+\sum\nolimits_{j\in\fml{C}}v_j^{\omega}
\end{equation}
We are interested in cases where $\Gamma^{\omega}\le0$, corresponding to
predicting $\ominus$ instead of $\oplus$. (Otherwise the prediction
would not change from $\oplus$.) The expression above can be rewritten
as follows,
\begin{equation}
  \begin{array}{rcl}
    \Gamma^{\omega} & = &
    w_0+\sum_{j\in\fml{C}}v_j^{a_j}-\sum_{j\in\fml{C}}(v_j^{a_j}-v_j^{\omega})\\[3.0pt]
    & = & \Gamma - \sum_{j\in\fml{C}}\delta_j = -\Phi \\ %- (\Delta-\Gamma) = 
  \end{array}
\end{equation}
where we use $\delta_j\triangleq{v_j^{a_j}}-{v_j^{\omega}}$,
%$\Delta\triangleq\sum_{j\in\fml{C}}\delta_j$,
and
$\Phi\triangleq\sum_{j\in\fml{C}}\delta_j-\Gamma=-\Gamma^{\omega}$.
%\nnote{$\delta_j$ is not defined}. -> Joao: both fixed
%\nnote{$\Delta$ is not defined an duded in the algorithms}.
%
Our goal is to find a smallest (or subset-minimal) set $\fml{P}$ such
that the prediction is still $\oplus$ 
(whatever the values of the other features):
%(i.e.\ we compute a smallest $\Gamma^t>0$):
%
\begin{equation} \label{eq:xlc04}
  %\Gamma^t = 
  w_0 + \sum\nolimits_{j \in \fml{P}} v_j^{a_j} + \sum\nolimits_{j \notin \fml{P}} v_j^{\omega} =
  -\Phi + \sum\nolimits_{j\in\fml{P}}\delta_j > 0
\end{equation}
i.e.\ we want to pick a smallest (or subset-minimal) set of literals
that ensures that the prediction will be $\oplus$.
In turn,~\eqref{eq:xlc04} can be represented as the following
optimization problem:
\begin{equation} \label{eq:xlc05}
  \begin{array}{lcl}
    \tn{min}  & \quad & \sum_{i=1}^{n}p_i \\[4.5pt]
    \tn{s.t.} & \quad &
    \sum_{i=1}^{n}\delta_ip_i>\Phi \\[2.5pt]
    & & p_i\in\{0,1\}\\
  \end{array}
\end{equation}
where the variables $p_i$ assigned value 1 denote the indices included
in $\fml{P}$.
Although solving~\eqref{eq:xlc05} seems to equate to solving an
NP-hard optimization, concretely the minimization version of the
knapsack problem~\cite{pisinger-bk04}, the fact that the coefficients
in the cost function are all equal to 1 makes the problem solvable in
log-linear time\footnote{%
  Pseudo-polynomial time algorithms for the knapsack problem are
  well-known~\cite{dantzig-or57,papadimitriou-bk82}. One concrete
  example~\cite{papadimitriou-bk82} yields a polynomial (cubic) time
  algorithm in the setting of computing a smallest PI-explanation of
  an XLC. We show that it is possible to devise a more efficient
  solution.}.
%
%\nnote{The enumeration problem looks to be related to \#Knapsack:
%\url{https://drops.dagstuhl.de/opus/volltexte/2018/9068/pdf/LIPIcs-ICALP-2018-64.pdf}} -> Joao: yes, but they look at the general problem.
%
%
%The constraint in~\eqref{eq:xlc06} can be written as follows:
%%
%\begin{equation} \label{eq:xlc07}
%\sum_{i=1}^{n}\lambda_ip_i>\Phi
%\end{equation}
%
%where $\lambda_i=\delta_i-\rho_i$ and $\Phi=-w_0-\sum_{i=1}^{n}\rho_i$.
%
Concretely, we can now develop a greedy algorithm that computes a
smallest PI-explanation, representing one optimal solution
of~\eqref{eq:xlc05}. At each step, we simply pick the largest
$\delta_i$ that has not yet been picked. 
\begin{repproposition}{prop:1xpl}
  Let $\fml{S}=\langle{l_1},\ldots,{l_n}\rangle$ represent indices
  of $\fml{E}$ sorted by non-increasing value of $\delta_j$.
  Pick $k$ such that
  $\sum_{j\in\{{l_1},\ldots,{l_k}\}}\delta_j>\Phi$ and
  $\sum_{j\in\{{l_1},\ldots,{l_{k-1}}\}}\delta_j\le\Phi$.
  Then~\eqref{eq:xlc03} holds for
  $\fml{P}=\{p_{l_r}|1\le{r}\le{k}\}$, and $\fml{P}$ represents an
  optimal solution of $\eqref{eq:xlc05}$.
\end{repproposition}
Optimality of the computed solution is given by~\autoref{prop:1xpl}
(proof included in~\autoref{app:proofs}).

%\begin{figure}[t]
%  \centering\input{./tabs/ex01d}
%  \caption{Computing $\delta_j$'s for XLC} \label{fig:ex03}
%\end{figure}

\begin{example} \label{ex:ex02}
  \autoref{fig:ex03} shows the values used for computing explanations
  for the example in~\autoref{fig:ex02}.\\
  For this example, the sorted $\delta_j$'s become
  $\langle\delta_1,\delta_2,\delta_4,\delta_3\rangle$.
  By picking $\delta_1$ and $\delta_2$, we ensure that the prediction
  is $\oplus$, independently of the evidence provided for features
  $e_3$ and $e_4$.
  Thus $(e_1)\land(\neg e_2)$ is a PI-explanation for the NBC shown
  in~\autoref{fig:ex01}, with evidence
  $(e_1,e_2,e_3,e_4)=(\lvt,\lvf,\lvt,\lvf)$.
  (It is easy to observe that
  $\tau(\lvt,\lvf,\lvf,\lvf)=\tau(\lvt,\lvf,\lvf,\lvt)=\tau(\lvt,\lvf,\lvt,\lvf)=\tau(\lvt,\lvf,\lvt,\lvt)=\oplus$).
\end{example}

\begin{algorithm}[t]
  {\relsize{\fnsqz}
    \input{./algs/onexpl}
  }
  \caption{Finding one explanation} \label{alg:1xpl}
\end{algorithm}

In the concrete case of NBCs, if the goal is to compute a single
explanation, then the algorithm detailed in this section is
exponentially more efficient (in the worst case) than earlier
work~\cite{darwiche-ijcai18}.
However, in some settings one wants to be able to analyze some or even
all explanations for a given instance (this is further discussed
in~\autoref{sec:res}). We describe next a polynomial (log-linear)
delay algorithm for enumeration of explanations for XLCs (and so for
NBCs).

\paragraph{Enumerating explanations with polynomial delay.}
As shown above, a smallest PI-explanation can be computed in log-linear
time by sorting the $\delta_i$ values and picking the first $k$
literals that ensure the prediction. We start by presenting a more
elaborate description of the algorithm, which we then use for devising
the enumeration of explanations with polynomial delay\footnote{%
  For a knapsack constraint, it is known that feasible solutions 
  can be enumerated with quadratic
  delay~\cite{lawler-sjc80,johnson-ipl88}. Nevertheless, we exploit 
  the problem's special structure to achieve a log-linear enumeration
  delay.}.
\autoref{alg:1xpl} shows the pseudo-code for computing one smallest
explanation. $\Delta$ denotes the array of sorted $\delta_j$'s. (The
pseudo-code assumes that the order $1,2,\ldots,n$ represents the
literals in sorted order.) $\Phi^{R}$ is initialized with the value of
$\Phi$, being updated as the algorithm(s) progress(es).
\autoref{alg:1xpl} corresponds to the direct application
of~\autoref{prop:1xpl}.
This algorithm can now be
%%for finding one smallest PI-explanation can be
exploited for implementing a polynomial delay algorithm for
enumerating PI-explanations.
%
%\nnote{The algorithm AllExplainations  is not very clear (to
%me). What is Idx iterating over? I wonder if we need to mention that
%the algorithm is complete in a sense that it terminates when it
%enumerated all solutions} -> Joao: fixed
%
%\nnote{A side comment: I am wondering that if $\Phi$ is not very
%large and $\delta_i$'s precision is small, e.g. 2, so we can assume
%we have integers, we can build DP and count the number of solutions
%(just the number of paths in DP graph).}
% -> Joao: interesting, something to look into, but most likely
% depends on the weights.
%
\autoref{alg:allxpl} depicts the enumeration of PI-explanations.
The algorithm implements a (restricted) backtrack search procedure,
which in some circumstances can be shown to yield polynomial delay 
algorithms~\cite{cohen-jc04}.
$\msf{Idx}$ denotes the depth of the search tree and $\msf{Flip}$ (if
assigned 0) records which $\delta_j$'s are used for updating
$\Phi^{R}$. (The entries of $\msf{Flip}$ take value -1 if unused, and
value 1 if have been backtracked upon.)
A key aspect of the algorithm is that it only branches when it is
guaranteed that a PI-explanation can still be found, given
the prefix (of picked or not picked $\delta_j$'s) defined by
$\msf{Flip}$ and $\msf{Idx}$. Otherwise, the algorithm must backtrack
and enter a consistent state (with at most a linear backtracking
effort). 
\begin{algorithm}[t]
  {\relsize{\fnsqz}
    \input{./algs/allxpl}
  }
  \caption{Finding all explanations} \label{alg:allxpl}
\end{algorithm}
%
%Besides the extraction of one PI-explanation, the algorithm must be
%able to enter a consistent state with at most a polynomial
%backtracking effort.
\autoref{alg:valst} shows the backtrack step of the PI-enumeration
algorithm.
\autoref{alg:valst} terminates if no more PI-explanations can be
found, or with the guarantee that another PI-explanation can be
extracted with \autoref{alg:1xpl}.
\begin{algorithm}[t]
  {\relsize{\fnsqz}
    \input{./algs/valst}
  }
  \caption{Entering a valid state} \label{alg:valst}
\end{algorithm}
It is straightforward to conclude that both \autoref{alg:1xpl} and
\autoref{alg:valst} run in linear time on the size of the current
depth of the search tree (which is linear on the number of features).
Thus, we can list PI-explanations of XLC's with polynomial delay
(proof included in~\autoref{app:proofs}).

\begin{repproposition}{prop:allxpl}
  PI-explanations of an XLC can be enumerated with log-linear delay.
\end{repproposition}

\jnoteF{Add more detail. For instance, explain how one can decide
  whether a solution can/cannot be found given the status of the
  search tree.}

\paragraph{Real-valued \& mixed case.}
Let us now consider $\fml{R}\not=\emptyset$. As before, the prediction
is assumed to be $\oplus$.
For each feature, if $w_i>0$, then we are interested in assessing the
impact of reducing the value of $e_i$. Hence, the worst-case scenario
is achieved when $e_i=\lambda_i$. In this case, we define
$\delta_i=(a_i-\lambda_i)w_i$. A no-change constraint on the value of
$e_i$ is formulated as $e_i\ge{a_i}$ (i.e.\ we \emph{clamp} the value
of $e_i$ by imposing a lower bound on its value).
In contrast, if $w_i<0$, then we are interested in assessing the impact
of increasing the value of $e_i$. The worst-case scenario is now
$e_i=\mu_i$. In this case, we define $\delta_i=(a_i-\mu_i)w_i$.
Moreover, a no-change constraint on the value of $e_i$ is
formulated as $e_i\le{a_i}$ (i.e.\ in this case we \emph{clamp} the
value of $e_i$ by imposing an upper bound on its value).
Given the definition of the $\delta_i$ constants for real-valued
features, and associated literals in case of a no-change constraint,
we can compute explanations using the restricted knapsack problem
formulation as above.
Thus, we can also compute one cardinality optimal solution in
log-linear time, and enumerate subset-minimal solutions with
polynomial delay.

%% file: tabs/ex01c.tex
\renewcommand{\tabcolsep}{0.25em}
\renewcommand{\arraystretch}{1.125}
\begin{tabular}{|c|cc|cc|cc|cc|} \hline  %\toprule
  $w_0$ &
  $v_1^1$ & $v_1^2$ &
  $v_2^1$ & $v_2^2$ &
  $v_3^1$ & $v_3^2$ &
  $v_4^1$ & $v_4^2$
  \\ \hline  %\midrule
  -2.19 &
  -2.97 & 3.46 &  %% @0 & @1 
  2.95 & -2.95 &
  0.4 & -2.83 &
  1.17 & -1.32
  \\ \hline  %\bottomrule
\end{tabular}
%% Data confirmed with the cnbc2xlc script... See file t02a-cnbc,
%% but also related t02a-xyz files

%% file: tabs/ex01d.tex
\renewcommand{\tabcolsep}{0.25em}
\renewcommand{\arraystretch}{1.125}
\begin{tabular}{|c|cccc|c|} \hline  %\toprule
  $\Gamma$ &    %%  2.56
  $\delta_1$ &  %%  6.43
  $\delta_2$ &  %%  5.90
  $\delta_3$ &  %%  0.00
  $\delta_4$ &  %%  2.49
  $\Phi$        %%  12.26
  \\ \hline  %\midrule
  2.56 &
  6.43 & 
  5.90 &
  0.00 &
  2.49 &
  12.26
  \\ \hline  %\bottomrule
\end{tabular}
%% Data confirmed with the cnbc2xlc script... See file t02a-cnbc,
%% but also related t02a-xyz files

%% file: algs/onexpl.tex
%------------------------------------------------------------------------------%
% File:        onexpl.tex
%------------------------------------------------------------------------------%
%
\input{./algs/alg-defs}

\Func \OneXpl{$\msf{Vs}$,$\msf{Flip}$,$\Delta$,$\Phi^{R}$,$\msf{Idx}$,$\msf{Xpl}$}
\;
%\LinesNumbered
\Indp
%\Global{ $\Pred$: Monotone predicate }
%; $\mathscr{P}$; $\mathcal{P}$; $\mathfrak{P}$
\KwIn{
  $\msf{Vs}$:~Values of instance being explained;
  $\msf{Flip}$:~Array reference of decision steps;
  $\Delta$:~Sorted $\delta_j$'s;
  $\Phi^{R}$:~Explanation threshold;
  $\msf{Idx}$:~Index for $\Delta$;
  $\msf{Xpl}$:~Set reference of explanation literals
}
\KwOut{
  $\Phi^{R}$: Updated threshold; 
  $\msf{Idx}$: Updated index for $\Delta$
}
\TopBlankLine
{
  \lnlset{xpl:1}{1}
  \While{$\Phi^{R}\ge0$}{
    \lnlset{xpl:2}{2}
    $\msf{Idx}\gets\msf{Idx}+1$ \;
    \lnlset{xpl:3}{3}
    $\msf{Flip}[\msf{Idx}] \gets 0$ \;
    \lnlset{xpl:4}{4}
    $\Phi^{R}\gets\Phi^{R}-\Delta[\msf{Idx}]$ \;
    \lnlset{xpl:5}{5}
    $\msf{Xpl}\gets\msf{Xpl}\cup\{(e_{\msf{Idx}},\msf{Vs}[\msf{Idx}])\}$ \;
  }
  \lnlset{xpl:6}{6}
  \ReportXPl($\msf{Xpl}$) \;
  \lnlset{xpl:7}{7}
  \Return ($\Phi^{R},\msf{Idx}$) \;
}
%\SetAlgoNoLine
\Indm
\BotBlankLine
%
%------------------------------------------------------------------------------%

%% file: algs/alg-defs.tex
%------------------------------------------------------------------------------%
% File:        alg-defs.tex
%
% Description: 
%
% Created:     10 Sep 2012.
%
% Author:      Joao Marques-Silva (jpms).
%------------------------------------------------------------------------------%
%
%\DontPrintSemicolon
\SetAlgoNoEnd
%\SetAlgoShortEnd
\SetAlgoNoLine
%\SetAlgoLined
%\SetAlgoVlined
%\LinesNumbered
%\SetAlgoNoEnd
%\SetSideCommentLeft
\SetFillComment
%\SetKwBlock{Let}{let}{end}
%\SetKwBlock{FBlock}{}{end}
\SetKwBlock{Let}{let}{end}
\SetKwBlock{FBlock}{}{}
\SetKw{KwNot}{not\xspace}
\SetKw{KwAnd}{and\xspace}
\SetKw{KwOr}{or\xspace}
\SetKw{Break}{break\xspace}
\SetKwData{false}{{\small false}}
\SetKwData{true}{{\small true}}
\SetKwData{st}{\small{\sl st}}
\SetKwFunction{OneXpl}{{\sc OneExplanation}} % 
\SetKwFunction{AllXpl}{{\sc AllExplanations}} % 
\SetKwFunction{ValSt}{\sc EnterValidState}
\SetKwFunction{ReportXPl}{\sc ReportExplanation}
\SetKwBlock{Let}{let}{end}
\SetKwBlock{FBlock}{}{end}
\SetKwInOut{Global}{Global}
%\SetKwInOut{Function}{function}
\SetKwHangingKw{Algorithm}{Algorithm}
\SetKwHangingKw{Function}{function}
\SetKw{Func}{Function}
\SetKwBlock{Begin}{}{}

%% file: algs/allxpl.tex
%------------------------------------------------------------------------------%
% File:        allxpl.tex
%------------------------------------------------------------------------------%
%
\input{./algs/alg-defs}
%
\Func \AllXpl{$\msf{Vs}$,$\Delta$,$\Phi^{R}$}
\;
%\LinesNumbered
\Indp
%\Global{ $\Pred$: Monotone predicate }
%; $\mathscr{P}$; $\mathcal{P}$; $\mathfrak{P}$
\KwIn{
  $\msf{Vs}$:~Values of instance being explained;
  $\Delta$:~Sorted $\delta_j$'s;
  $\Phi^{R}$:~Explanation threshold %;
  %$\msf{Flip}$:~Array reference of decision steps;
  %$\msf{Idx}$:~Index for $\Delta$;
  %$\msf{Xpl}$:~Set reference of explanation literals
}
%\KwOut{
  %$\Phi$: Updated threshold; 
  %$\msf{Idx}$: Updated index for $\Delta$
%}
\TopBlankLine
{
  \lnlset{axpl:1}{1}
  $(\msf{Xpl},\msf{Flip},\msf{Idx})\gets(\emptyset,[-1,\ldots,-1],0)$ \;
  \lnlset{axpl:2}{2}
  \While{$\msf{Idx}\ge0$}{
    \lnlset{axpl:3}{3}
    $(\Phi^{R},\msf{Idx})\gets\OneXpl(\msf{Vs},\msf{Flip},\Delta,\Phi^{R},\msf{Idx},\msf{Xpl})$ \;
    \lnlset{axpl:4}{4}
    $(\Phi^{R},\msf{Idx})\gets\ValSt(\msf{Vs},\msf{Flip},\Delta,\Phi^{R},\msf{Idx},\msf{Xpl})$ \;
  }
}
%\SetAlgoNoLine
\Indm
\BotBlankLine
%
%------------------------------------------------------------------------------%

%% file: algs/valst.tex
%------------------------------------------------------------------------------%
% File:        valst.tex
%------------------------------------------------------------------------------%
%
\input{./algs/alg-defs}
%
\Func \ValSt{$\msf{Vs}$,$\msf{Flip}$,$\Delta$,$\Phi^{R}$,$\msf{Idx}$,$\msf{Xpl}$}
\;
%\LinesNumbered
\Indp
%\Global{ $\Pred$: Monotone predicate }
%; $\mathscr{P}$; $\mathcal{P}$; $\mathfrak{P}$
\KwIn{
  $\msf{Vs}$:~Values of instance being explained;
  $\msf{Flip}$:~Array reference of decision steps;
  $\Delta$:~Sorted $\delta_j$'s;
  $\Phi^{R}$:~Explanation threshold;
  $\msf{Idx}$:~Index for $\Delta$;
  $\msf{Xpl}$:~Set reference of explanation literals
}
\KwOut{
  $\Phi^{R}$: Updated threshold; 
  $\msf{Idx}$: Updated index for $\Delta$
}
\TopBlankLine
{
  %%\lnlset{vst:1}{1}
  %%$\msf{StOk}\gets\false$ \;
  %%\lnlset{vst:2}{2}
  %%\While{\KwNot~$\msf{StOk}$}{
  \lnlset{vst:1}{1}
  \While{$\Phi^{R}<0$ \KwOr %$\msf{Idx}=n$ \KwOr
    $\sum_{i=\msf{Idx}}^{n}\Delta[i]<\Phi^{R}$}{
    \lnlset{vst:2}{2}
    \While{$\msf{Idx}\ge0\land\msf{Flip}[\msf{Idx}]=1$}{
      \lnlset{vst:3}{3}
      $\msf{Flip}[\msf{Idx}]\gets-1$ \;
      \lnlset{vst:4}{4}
      $\msf{Idx}\gets\msf{Idx}-1$ \;
    }
    \lnlset{vst:5}{5}
    \lIf{$\msf{Idx}<0$}{ \Return $(\Phi^{R},\msf{Idx})$ }
    \lnlset{vst:6}{6}
    $\msf{Xpl}\gets\msf{Xpl}\setminus\{(e_{\msf{Idx}},\msf{Vs}[\msf{Idx}])\}$\;
    \lnlset{vst:7}{7}
    $\Phi^{R}\gets\Phi^{R}+\Delta[\msf{Idx}]$\;
    \lnlset{vst:8}{8}
    $\msf{Flip}[\msf{Idx}] \gets 1$\;
    %%\lnlset{vst:10}{10}
    %%$\msf{rIdx}\gets\msf{Idx}$\;
    %%\lnlset{vst:11}{11}
    %%\lIf{$\msf{Flip}[\msf{Idx}]\not=-1$}{$\msf{rIdx}\gets\msf{rIdx}+1$}
    %%\lnlset{vst:10}{10}
    %%\lIf{$\sum_{i=\msf{rIdx}}^{n}\Delta[i]\ge\Phi$}{$\msf{StOk}\gets\true$}
  }
  \lnlset{vst:9}{9}
  \Return ($\Phi^{R},\msf{Idx}$) \;
}
%\SetAlgoNoLine
\Indm
\BotBlankLine
%
%------------------------------------------------------------------------------%

%% file: res.tex
\section{Experimental Evaluation} \label{sec:res}
This section evaluates the PI-explanation enumerator XPXLC, that
implements the algorithms described in this paper%
%implemented as a set of Perl scripts
%\footnote{The source code of XPXLC as well as the datasets, a demo and
%  accompanying  documentation are available at
%\url{https://github.com/expxlc/expxlc}.}.
\footnote{The source code of XPXLC, as well as the datasets,
  documentation, and additional examples can be obtained from the
  authors.}.
XPXLC was tested in Debian Linux on an Intel~Xeon~CPU~5160~3.00\,GHz
with 64\,GByte of memory.
When testing scalability, XPXLC was run with 8GByte limit on RAM and
two hours time limit.
The experiment was divided into 3 parts:
(1)~evaluating the raw performance of XPXLC,
(2)~comparing it with the state-of-the-art compilation
approach STEP~\cite{darwiche-ijcai18,darwiche-aaai19}, and
(3)~using complete enumeration of PI-explanations to assess the
quality of explanations of the well-known heuristic explainers
Anchor~\cite{guestrin-aaai18} and SHAP~\cite{lundberg-nips17}.

% \subsection{Datasets}
%
% \paragraph{Datasets.}
\noindent\textbf{Datasets.}
We selected a set of widely-used, publicly available, datasets
from~\cite{uci,pennml,kaggle}.
% UCI Machine Learning Repository~\cite{uci}, Penn Machine Learning
% Benchmarks~\cite{pennml} and Kaggle Machine Learning
% Community~\cite{kaggle}.
%
The total number of datasets used is 37.
For each dataset, we trained a Naive Bayes classifier\footnote{The
  CategoricalNB classifier of scikit-learn~\cite{sklearn} was used for
  this purpose.} using 80\% of the training data. 
The average test accuracy assessed for the 20\% remaining instances is
77.7\%.
(All the datasets and the trained classifiers are available in the
online repository.)
The experiments targeted XPXLC's ability to enumerate a given number
of explanations within a time limit.

% \paragraph{Raw performance.}
\noindent\textbf{Raw performance.}
\autoref{fig:cactus} shows the scalability of XPXLC.
Here, XPXLC was set to compute $\text{10}^\text{6}$ distinct
explanations for each instance of each dataset.
For the cases having fewer than $\text{10}^\text{6}$ explanations,
%with a number of explanations less than
XPXLC terminates as soon as all explanations are computed.
The smallest number of observed explanations per instance is 1, the
maximum number is at least $\text{10}^\text{6}$, while on average
29207.5 PI-explanations are reported per each instance.
The total number of instances to explain in this experiment is 94174.
The line drawn through point $(x, y)$ in \autoref{fig:cactus} shows
how many instances on the $X$-axis are solved by the time shown on the
$Y$-axis.
As can be observed, performance is not an issue for XPXLC -- it never
exceeds $\text{12}$ seconds to enumerate $\text{10}^\text{6}$
%\emph{a million}
explanations for each of the target instances.
On average, XPXLC finishes complete enumeration (of at most
$\text{10}^\text{6}$ explanations) in 0.23 seconds.

\begin{figure*}[!t] \centering
  \begin{subfigure}[b]{0.315\textwidth}
    \centering
    \includegraphics[width=\textwidth]{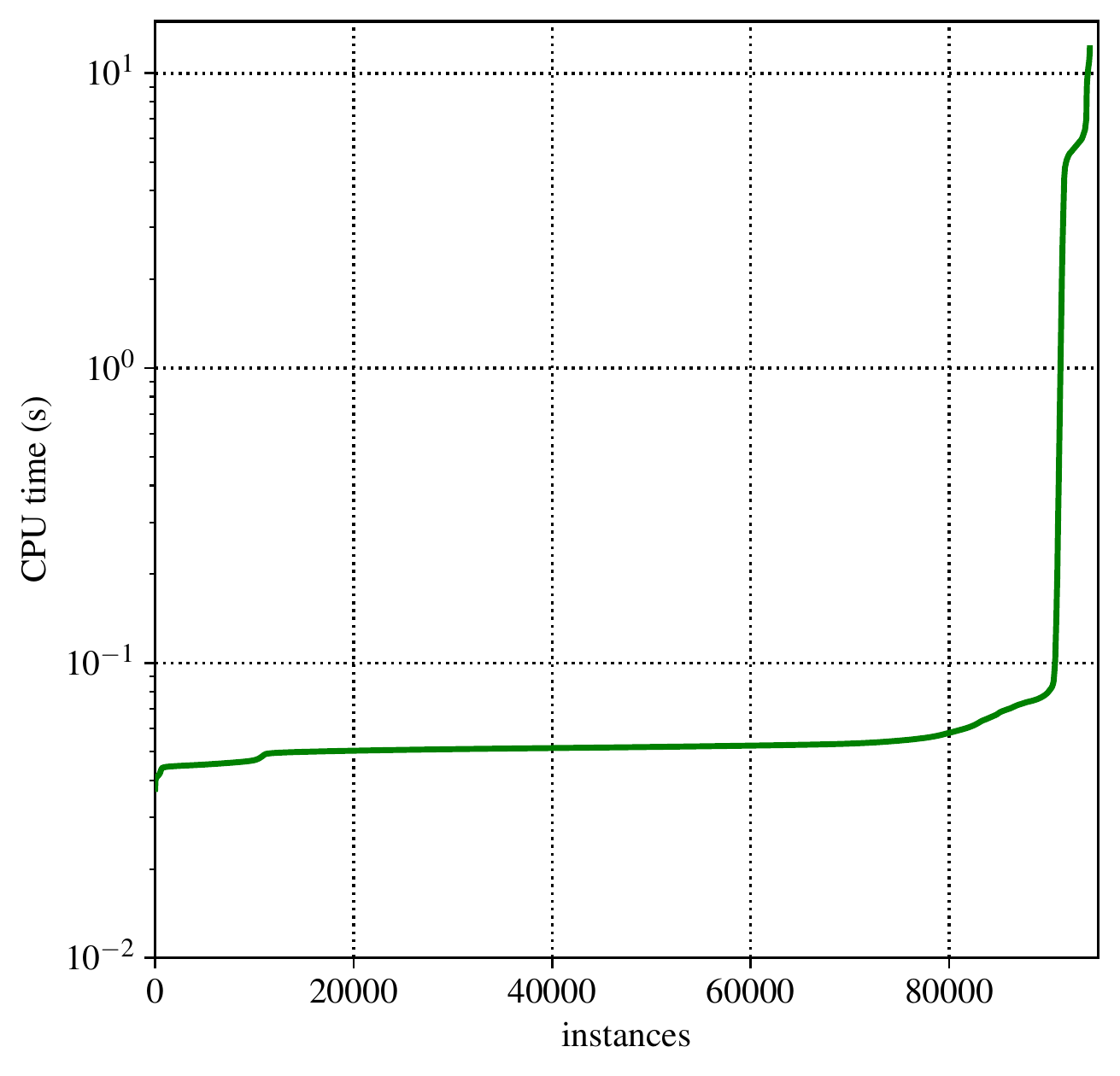}
    \caption{Raw performance of XPXLC}
      % for $\text{10}^\text{6}$ PI-explanations}
    %
    \label{fig:cactus}
  \end{subfigure}%
  \hspace{3pt}
  \begin{subfigure}[b]{0.315\textwidth}
    \centering
    \includegraphics[width=\textwidth]{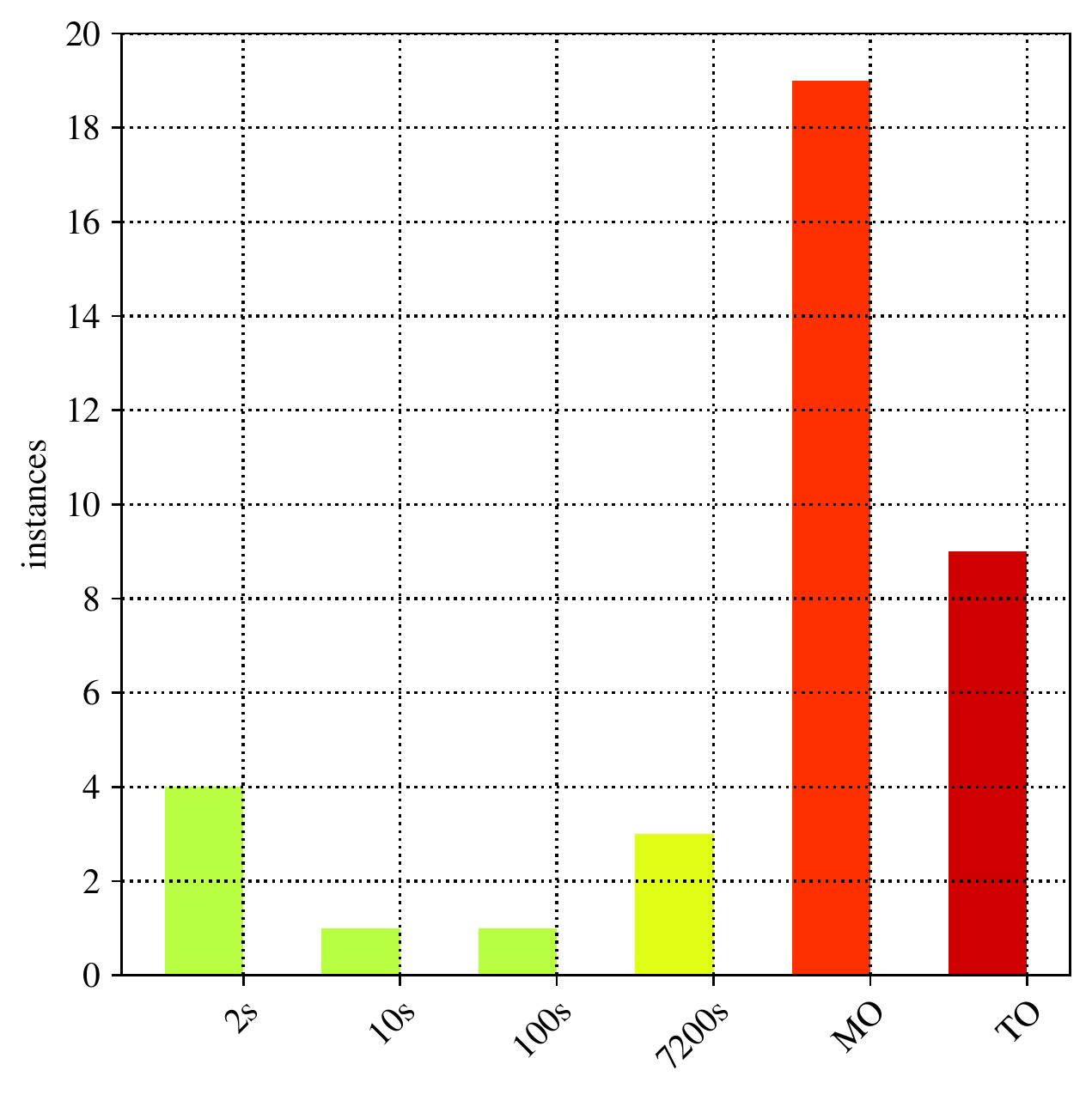}
    \caption{Performance of STEP (with MOs \& TOs)}
    \label{fig:hist}
  \end{subfigure}
  \hspace{3pt}
  \begin{subfigure}[b]{0.325\textwidth}
    \centering
    \includegraphics[width=\textwidth]{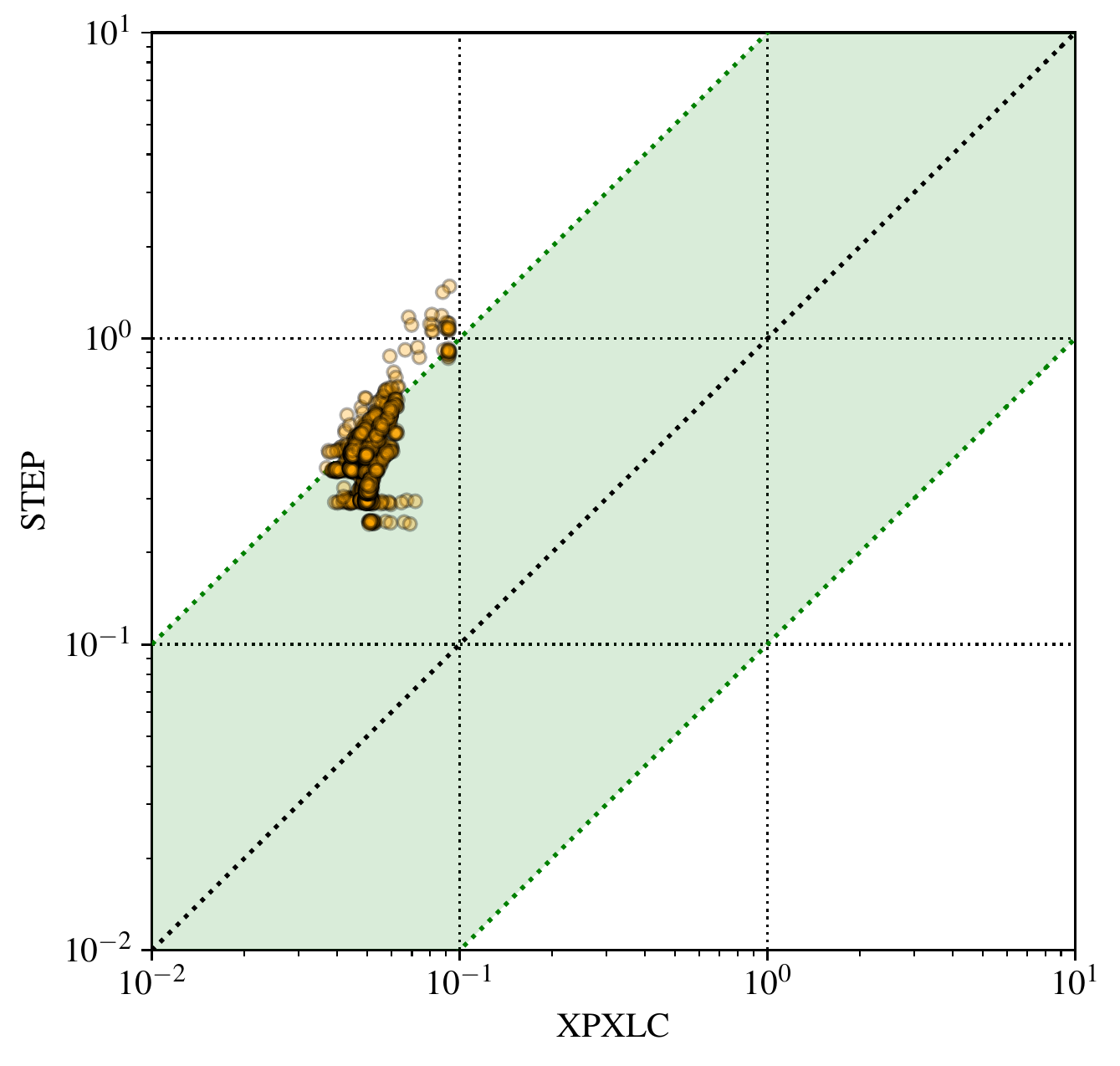}
    \caption{XPXLC vs STEP (no comp.~time)}
    \label{fig:scatter-nocomp}
  \end{subfigure}
  \caption{Scalability of XPXLC targeting $\text{10}^\text{6}$
  PI-explanations, performance of STEP, and comparative performance of
XPXLC and STEP.}
  \label{fig:plots}
\end{figure*}

% \jnote{Summary of results on the datasets we have access to. We should
% focus on datasets where one hot encoding (OHE) is not used, but also on
% categorical examples. For OHE datasets, we should include a disclaimer.}

%\subsection{Enumerative vs.\ Compilation Approaches}
%\paragraph{Enumerative vs.\ compilation-based approaches.}
\noindent\textbf{Enumerative vs.\ compilation-based approaches.}
The state of the art for finding PI-explanations for NBCs is
the STEP compilation-based
approach~\cite{darwiche-ijcai18,darwiche-aaai19,ucla}.
Concretely, STEP consists of (1)~compilation of a BNC classifier into
a \emph{sentential decision diagram} (SDD) and (2)~enumeration of
PI-explanations using efficient algorithms for SDD-based prime
implicant enumeration.
The existing implementation of STEP can only handle binary features.
Therefore, and in order to compare the relative performance of XPXLC
and STEP, we apply a one-hot encoding (OHE) to categorical features,
retrain the Naive Bayes classifiers and run both tools on the
OHE instances\footnote{%
This solution is not ideal, since the use of OHE impacts the
assumption of feature independence of NBCs, and only serves to enable
the comparison between STEP and XPXLC.}, targeting the complete
enumeration of explanations.
Moreover,
despite its worst-case exponential complexity in time and space, STEP
can still compile into SDDs 9 (out of 37) NBC classifiers, i.e.\ close
to 25\% of the classifiers, within the 2 hours time limit and 32 GByte
memory limit.
%
%STEP is worst-case exponential both in time and space, and
%this is epitomized by the fact that, within the 2 hours time limit
%and 32 GByte memory limit, STEP can only compile into SDDs 9 (out of
%37) NBC classifiers,
%i.e.\ slightly less than 25\% of the classifiers.
%
Once an NBC classifier is compiled into an SDD, enumeration of all
PI-explanations is relatively easy --- concretely, it takes 0.39
seconds for the compilation-based approach to enumerate all
explanations.
However, the SDD compilation step itself takes between 1 and 4300
seconds for the classifiers that can be compiled.
If the compilation time is amortized over all data instances of each
dataset, its impact ranges from a fraction of a second to $\approx$50
seconds.
\autoref{fig:hist} shows a histogram summarizing the performance of
STEP's compiler.
The bars in the histogram represent the classifiers that STEP is able
to compile within 2 seconds (there are 4 of them), 10 seconds
(1), 100 seconds (1), 2 hours (3) and also classifiers that STEP
fails to compile due to reaching the memory (MO) or time (TO) limits.
The last two bars represent 19 and 9 classifiers, respectively.
% The last two bins contain 18429 and 25088 instances, respectively.
%
%
% For the same instances, the new approach outperforms compilation by
% a few orders of magnitude.
%
%\autoref{fig:scatter} shows a scatter plot, which details the
%comparison of XPXLC targeting $\text{10}^\text{6}$ PI-explanations
%vs.\ STEP (on the same number of target PI-explanations).
%%
%Here, a point $(x,y)$ represents running time (in seconds) spent by
%XPXLC (shown on the $X$-axis) and by compilation (shown on the
%$Y$-axis) for a concrete data instance.
%%
%If a point lies on one of the two \emph{timeout} borders, the
%corresponding approach is timed out for a given data instance.
%%
%Note that the comparison here takes into account the time spent for
%both compilation and explanation enumeration.
%%
%Observe that for a number of data instances XPXLC finishes the
%enumeration within $\leq \text{0.1}$ seconds while the compilation
%approach takes almost 2 hours to finish.
%%
%Also, the average running time of XPXLC on the complete set of data
%instances is 5.51 seconds and the maximum time for computing a
%million PI-explanations for a OHE instance is $\leq \text{300}$
%seconds.
%%
%This indicates that OHE affects the average runtime of XPXLC, which
%increases from 0.23 (see above) to 5.51 if OHE is applied.
%
%
Finally, \autoref{fig:scatter-nocomp} summarizes the performance
comparison between XPXLC and STEP. In this comparison, the SDD
compilation time is \emph{ignored}, and the plot shows only instances
for the classifiers that STEP is able to compile within the 2 hour
time limit.
Also note that both tools finish complete enumeration of
PI-explanations for each of these instances.
A point $(x,y)$ in the plot represents the time (in seconds) spent by
XPXLC (shown on the $X$-axis) and by STEP (shown on the $Y$-axis) for
a concrete data instance.
Observe that, even if the compilation time is ignored, STEP's
enumeration phase is still between 4 and 20 times slower than XPXLC.

%\paragraph{Finding One Explanation.}
%\paragraph{Finding $K$ Explanations.}
%\paragraph{Finding All Explanations.}

% \jnote{Compare with STEP~\cite{darwiche-ijcai18} on OHE datasets,
%   again with a disclaimer.}

%\subsection{Comparison with Heuristic Explanations}
%\paragraph{Heuristic vs. non-heuristic approaches.}
%\paragraph{Assessing heuristic approaches.}
\noindent\textbf{Assessing heuristic approaches.}
Exhaustive enumeration of PI-explanations can serve to assess
heuristic explanations.
Exhaustive enumeration provides a distribution of how many times
feature-value pairs appear in explanations, and thus which are likely
to be more \emph{relevant} for the given prediction.
As a result, one can evaluate how many features in
a heuristic explanation ``hit'' the set of most relevant (commonly-occurring) features.
This strategy may be beneficial in some practical settings where
trustable explanations are of concern.
%
%Despite being typically longer than heuristic explanations (and thus
%harder to interpret), PI-explanations (if enumerated effectively) hint
%on the set of important features and so help a user assess how good a
%heuristic explanation is.
%
While our ``hit'' metric is a heuristic evaluation measure to compare the quality of explanations, we demonstrate  its usefulness experimentally. For example, our metric does show a strong correlation between features of heuristic explanations and common features that we identify via enumeration.
\autoref{fig:heur} depicts the percentage of features in explanations
of Anchor~\cite{guestrin-aaai18} and SHAP~\cite{lundberg-nips17}
``hitting'' the set of common features.
Here, we focus on 2 datasets
\emph{Adult}~\cite{kohavi-kdd96,guestrin-aaai18} and
\emph{Spambase}~\cite{uci} and use the following methodology.
For an explanation $E$ of Anchor, we keep the top
% $\max(\text{5},|\fml{E}|)$ of features most frequently appearing in
$|E|$ features most commonly-occurring in all
PI-explanations\footnote{If $>|E|$ features are in the top due to
  having the same frequency, all of them are marked as common.
  Also, the experiment is performed only for instances for which
complete PI-explanation enumeration finishes.}; then we count the
number of features in $E$ that hit the set of common features.
As SHAP assigns numerical weights to \emph{all} features, we take 5
features reported by SHAP as most relevant and count how many of them
intersect the set of 5 most common features of PI-explanations.
The rationale of this choice is that larger explanations are typically
harder for a user to reason about and so $\text{5}$ features is
normally deemed enough to make a conclusion wrt.\ the cause of
prediction.
As can be observed, both Anchor and SHAP are successful at hitting the
most common features.
However, in some cases both tools' explanations do not overlap our important features, e.g.\ Anchor has  zero overlap with the common features
in more than 2000 instances. Given a significant overlap in the majority of cases, a zero hit suggests that Anchor's explanation might be using less influential features and is hence less
trustworthy. %a sign to the user to further inspect this case.
This experiment illustrates another setting where PI-explanations can
be useful, i.e.\ not only to output a provably correct
explanation but also to provide the user with an alternative evaluation toolkit to measure confidence in %inspect
heuristic explanations.
Finally, we observe that both Anchor and SHAP are significantly slower
than XPXLC: on average, Anchor takes 1.55 seconds to compute one
explanation of an instance, whereas SHAP takes 99.58 seconds.
In contrast, as highlighted above, XPXLC never exceeds a few tens of
$\mu$sec for computing a single explanation.
%
%This is in clear contrast with XPXLC, whose performance was analyzed
%above.

% \jnote{Questions to address:
%   \begin{enumerate}
%   \item Can Anchor really claim~\cite{guestrin-aaai18}: ``A set of
%     literals is a sufficient condition for $f(x)$ with high
%     probability``?
%     \begin{enumerate}
%     \item Can we estimate in how many points intersecting Anchor's
%       pseudo-explanation does entailment hold?
%     \end{enumerate}
%   \item Does Anchor even approximate what it proposes to compute?
%   \end{enumerate}
% }

%% file: conc.tex
\section{Conclusions} \label{sec:conc}

This paper presents a log-linear algorithm for computing
a smallest PI-explanation of linear classifiers. Moreover, the paper
shows that PI-explanations for linear classifiers can be enumerated
with polynomial delay. The results in the paper also apply to NBCs
(among other classifiers), and so should be contrasted with earlier 
work~\cite{darwiche-ijcai18}, which proposes a worst-case exponential
time and space solution for computing PI-explanations of NBCs.
A natural line of research is to investigate extensions of XLCs that
also admit polynomial time algorithms for computing PI-explanations.

%% file: appendix.tex
\appendix

\section{Appendix} \label{app:all}

%\pagebreak
\subsection{Additional Plots} \label{app:plots}

Additional plots are shown in~\autoref{fig:heur}.

\begin{figure*}[t]
  \centering
  \begin{subfigure}[b]{0.35\textwidth}
    \centering
    \includegraphics[width=\textwidth]{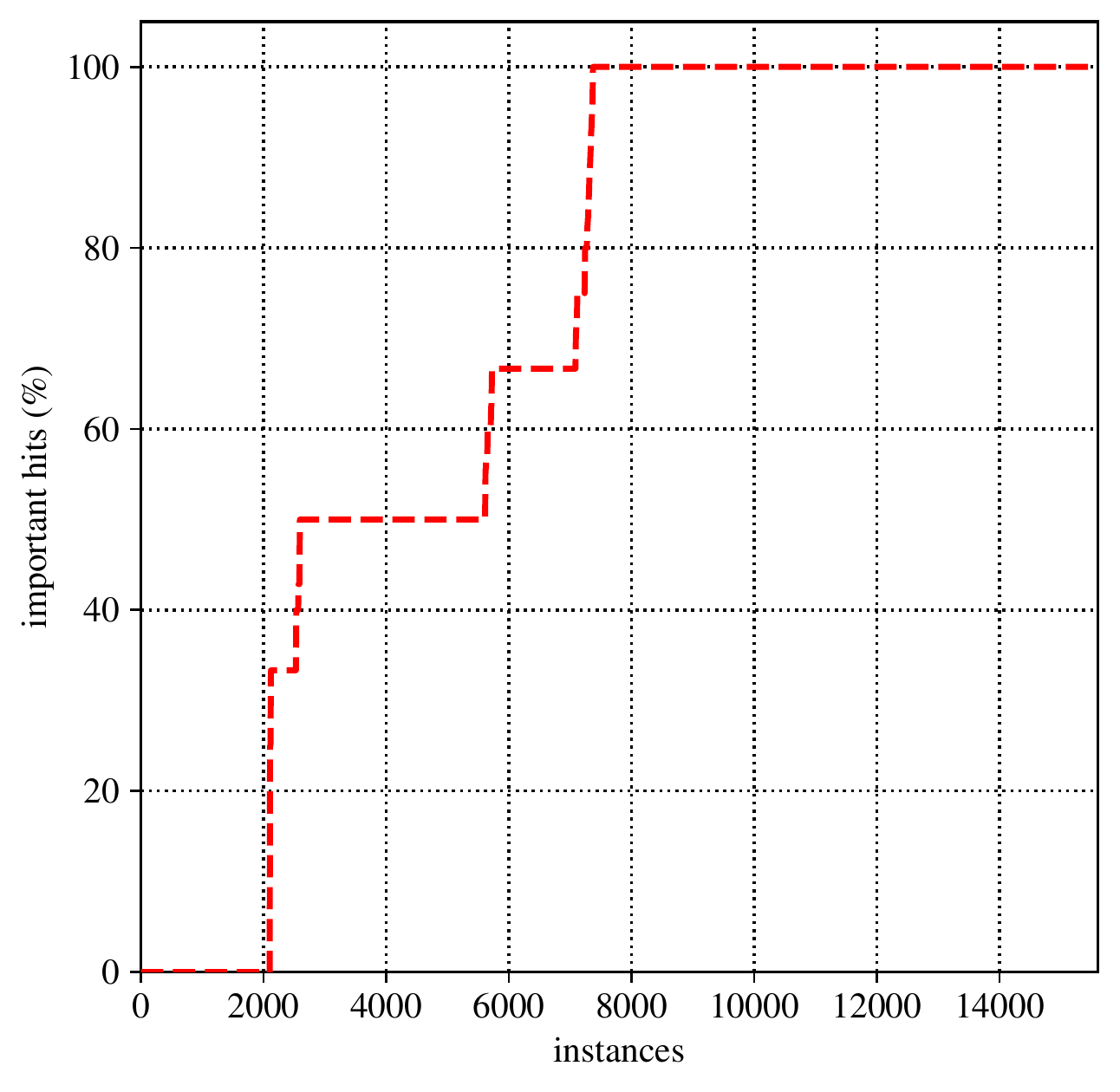}
    \caption{Anchor}
    \label{fig:anchor}
  \end{subfigure}%
  \hspace{50pt}
  \begin{subfigure}[b]{0.35\textwidth}
    \centering
    \includegraphics[width=\textwidth]{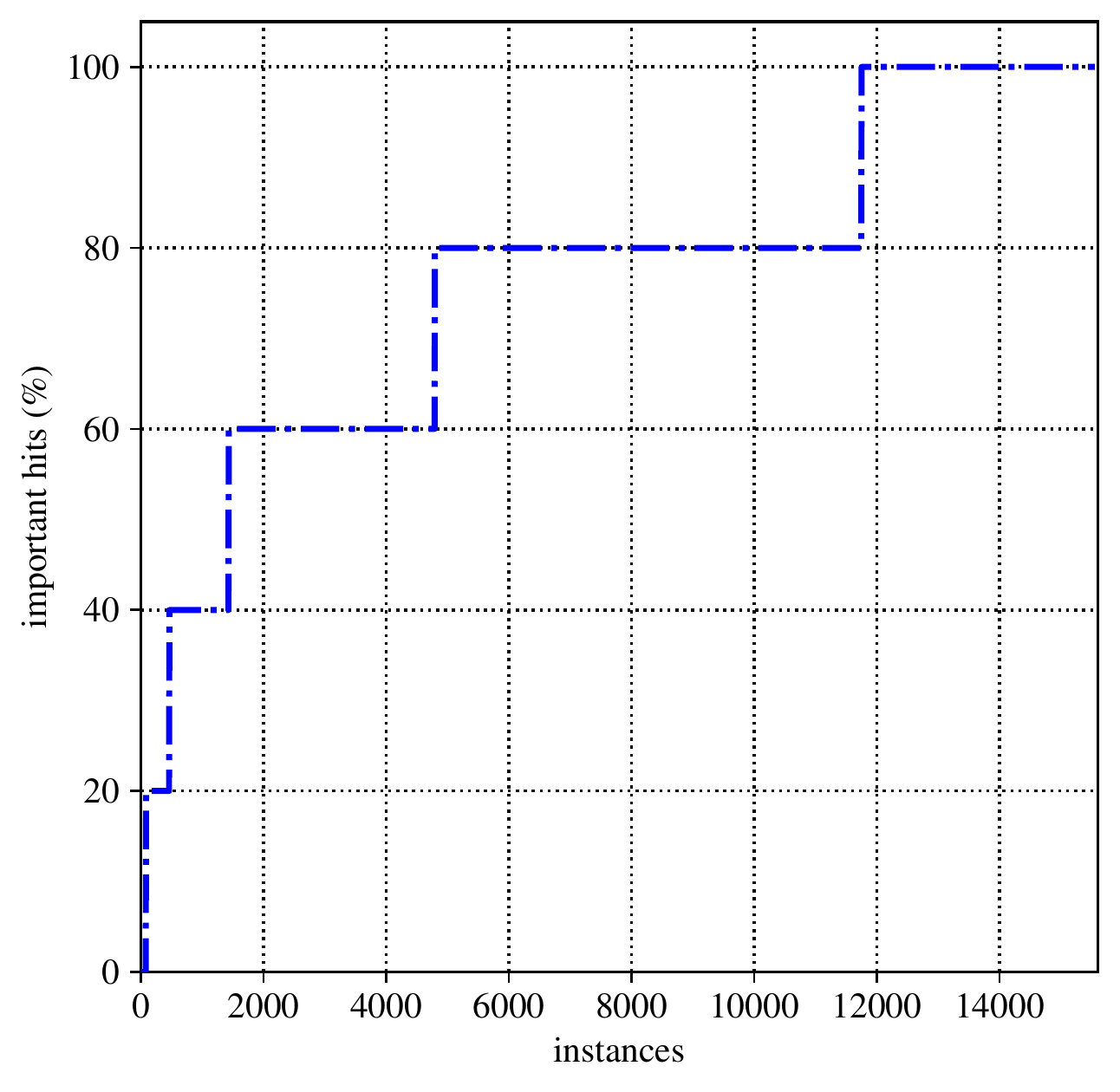}
    \caption{SHAP}
    \label{fig:shap}
  \end{subfigure}
  \caption{Percentage of important ``hits'' of explanations produced
  by Anchor and SHAP.}
  \label{fig:heur}
\end{figure*}

% (...)

\subsection{Proofs} \label{app:proofs}

\begin{proposition} \label{prop:1xpl}
  Let $\langle{l_1},\ldots,{l_n}\rangle$ represent indices
  $\fml{E}$ sorted by non-increasing value of $\delta_j$.
  Pick $k$ such that
  $\sum_{j\in\{{l_1},\ldots,{l_k}\}}\delta_j>\Phi$ and
  $\sum_{j\in\{{l_1},\ldots,{l_{k-1}}\}}\delta_j\le\Phi$.
  Then~\eqref{eq:xlc03} holds for
  $\fml{P}=\{l_r|1\le{r}\le{k}\}$, and $\fml{P}$ represents an
  optimal solution of $\eqref{eq:xlc05}$.
\end{proposition}

\begin{proof} %(Sketch)
  We prove that an optimal solution to~\eqref{eq:xlc05} can be
  obtained with the greedy algorithm that picks features in
  non-increasing order of $\delta_j$'s.
  Let
  $\fml{P}^{\ast}=\langle i_1,\ldots,i_k \rangle$
  denote the $k$ indices in some optimal solution, such that
  $\delta_{i_1}\ge\ldots\ge\delta_{i_k}$.
  Moreover, let
  $\mbb{V}(\fml{P}^{\ast})=\sum_{j\in\{{i_1},\ldots{i_k}\}}\delta_j$.
  Clearly,
  $\mbb{V}(\fml{P}^{\ast})>\Phi$; otherwise $\fml{P}^{\ast}$ would not
  satisfy the constraint in~\eqref{eq:xlc05}.

  %%$\mbb{V}(\fml{S})=\sum_{j\in\{{l_1},\ldots,{l_k}\}}\delta_j$ and
  %
  We prove by induction that one can construct another optimal
  solution
  $\fml{P}=\langle{l_1},\ldots,{l_k}\rangle$, where $l_1,\ldots,l_k$
  denote the first $k$ features with highest $\delta_j$.
  For the base case, we consider the first pick, and suppose
  that $i_1\not=l_1$ (and so $l_1$ does not occur 
  in $\fml{P}^{\ast}$).
  We can construct another sequence
  $\fml{P}'=\langle{l_1},i_2,\ldots,i_k\rangle$, such that
  $\mbb{V}(\fml{P}')=\sum_{j\in\{l_1,i_2,\ldots,i_k\}}\delta_j\ge\mbb{V}(\fml{P}^{\ast})>\Phi$.
  Hence, $\fml{P}'$ is still an optimal solution, and starts with a
  greedy choice.
  For the general case, we assume that the first $r{-}1$ picks can be
  made to respect the greedy choice, and that the $r^{\tn{th}}$ does
  not. The reasoning now can be mimicked again, and so we can
  construct another optimal solution such that the $r^{\tn{th}}$
  choice is also greedy.
  Thus,~\autoref{prop:1xpl} yields a smallest PI-explanation.
  %
  %
  %We prove that the greedy choice at each step of the algorithm finds
  %an optimal solution.
  %%
  %There is optimal solution to~\eqref{eq:xlc05} that picks
  %$l_1$. Why? Consider an optimal solution that starts with
  %$l_j$, with $j\not=1$. Then swap $l_j$ by $l_1$, and we
  %still have an optimal solution to~\eqref{eq:xlc05}.
  %By induction on the number of (greedy) picks, we conclude
  %that~\autoref{prop:1xpl} yields a smallest PI-explanation.
\end{proof}

\begin{proposition} \label{prop:allxpl}
  PI-explanations of an XLC can be enumerated with log-linear delay.
\end{proposition}

\begin{proof}
For simplicity of presentation, we assume that the values $\delta_i$
are sorted in non-increasing order, i.e.\ $\delta_1 \geq \ldots \geq
\delta_n$. This sorting operation can be achieved in log-linear 
time. Recall that $\delta_i \geq 0$ ($i=1,\ldots,n$) and
that a PI-explanation represented by the bit vector $p$ must satisfy
the two constraints: (C1) $\sum_{i=1}^{n} \delta_i p_i > \Phi$ and
(C2) $\forall j \in \{1,\ldots,n\}$ such that $p_j=1$,
$(\sum_{i=1}^{n} \delta_i p_i) - \delta_j p_j \leq \Phi$
(subset-minimality).

Consider an exhaustive depth-first binary search (DFS) in which at
depth $r$ the two branches correspond to $p_r=1$ and $p_r=0$.
It is critical for the correctness of this search that on each branch,
the $p_i$ variables are instantiated in non-increasing order of the
corresponding values $\delta_i$.
For a depth-$r$ node $\alpha$ of this search tree, let $S_{\alpha}$ be the
sum $\sum_{i=1}^{r} \delta_i p_i$. A node $\alpha$ is declared a leaf
(and is hence not expanded) if $S_{\alpha} > \Phi$. Assuming that, by
default, the remaining values $\delta_{r+1},\ldots,\delta_n$ are
assigned 0, node $\alpha$ satisfies (C1). Clearly, any other
descendant nodes (at which at least one of
$\delta_{r+1},\ldots,\delta_n$ is 1) would not satisfy (C2) and hence
does not need to be considered. This means that all PI-explanations will be found.
It remains to show that all leaves $\alpha$ satisfy subset-minimality and
hence are PI-explanations.
To see that $\alpha$ satisfies (C2),
let $\beta$ be its parent node. Since $\beta$ is not a leaf, we must
have $S_{\beta} = S_{\alpha} - \delta_r p_r \leq \Phi$.  But then
$S_{\alpha} - \delta_j p_j \leq \Phi$ for all $j$ such that $p_j=1$
since $\delta_j \geq \delta_r$ ($j=1,\ldots,r{-}1$). Thus, all leaves
correspond to PI-explanations.

We add to our DFS the pruning rule that a depth-$r$ node $\alpha$ is
only created if  $S_{\alpha} + \sum_{i=r{+}1}^{n} \delta_i >
\Phi$. This sum is calculated incrementally, so only requires $O(1)$
time at each node. The reason behind this rule is that if it is not satisfied,
then no descendant of $\alpha$ can satisfy (C1).  On the other hand,
if this rule is satisfied then we know that at least one descendant of
$\alpha$ will be a leaf (and as explained above will correspond to a
PI-explanation). It is well known that a depth-first search in a
search tree with no dead-end nodes provides a polynomial delay
algorithm~\cite{cohen-jc04}. In our DFS, the delay between visiting two leaves 
is linear in $n$. Since finding the first PI-explanation also requires a sorting step, 
with a log-linear complexity, we can conclude that the worst-case delay is log-linear. 
\end{proof}